\documentclass[conference]{IEEEtran}
\IEEEoverridecommandlockouts

\usepackage{cite}
\usepackage{amsmath,amssymb,amsfonts}
\usepackage{mathtools}
\usepackage{mathrsfs}
\usepackage{graphicx}
\usepackage{textcomp}
\usepackage{xcolor}
\usepackage{booktabs}
\usepackage{nicefrac}
\usepackage{microtype}
\usepackage{url}
\usepackage{physics}
\usepackage{bm}
\usepackage[linesnumbered,ruled]{algorithm2e}
\usepackage{bbm}
\usepackage{dsfont}

\def\BibTeX{{\rm B\kern-.05em{\sc i\kern-.025em b}\kern-.08em
    T\kern-.1667em\lower.7ex\hbox{E}\kern-.125emX}}

\newtheorem{theorem}{Theorem}

\newtheorem{lemma}[theorem]{Lemma}
\newtheorem{corollary}[theorem]{Corollary}

\begin{document}

\title{Coordinating the Unknown Lipschitz Constant in Multiplayer Bandits}

\author{\IEEEauthorblockN{Ricardo Parada}
\IEEEauthorblockA{\textit{University of California, Riverside}\\
Riverside, CA, USA \\
rparadaumanzor@gmail.com}
\and
\IEEEauthorblockN{Chenzhang Zhao}
\IEEEauthorblockA{\textit{University of California, Los Angeles}\\
Los Angeles, CA, USA \\
zhaochenzhang03@g.ucla.edu}
\and
\IEEEauthorblockN{William Chang}
\IEEEauthorblockA{\textit{University of California, Los Angeles}\\
Los Angeles, CA, USA \\
chang314@g.ucla.edu}
}

\maketitle

\begin{abstract}
Motivated by decentralized applications, we study cooperative multi-agent bandits in continuous (Lipschitz) action spaces when the Lipschitz constant is unknown. We consider three information structures: (A)~unobserved actions with common rewards, (B)~observed actions with independent rewards, and (C)~unobserved actions with independent rewards. In each case we design and analyze an algorithm that estimates the Lipschitz constant, chooses a discretization of the joint action space, and applies a cooperative bandit method to the induced discrete problem. Players never communicate once learning starts, so the central difficulty is that they must reach the \emph{same} discretization from their own data. We prove regret guarantees showing that common rewards and observable actions each supply this agreement for free, and that in their absence agreement can still be bought, through a dithered quantization of the estimate, at no cost in the leading order of the regret.
\end{abstract}


\section{Introduction}

The multi-armed bandit (MAB) problem is a central model in sequential decision-making: an agent repeatedly chooses among actions with unknown rewards, balancing exploration and exploitation to maximize cumulative reward. Many deployments are inherently distributed, with several decision-makers acting concurrently and observing only part of the system feedback. In wireless systems, multiple users opportunistically access channels whose availabilities are unknown and time-varying, often without explicit coordination, as in cognitive radio and dynamic spectrum access~\cite{liu2010distributed,anandkumar2011distributed}. Similar abstractions arise in distributed sensing and radar networks, where nodes select waveforms or frequency bands while attempting to avoid mutual interference~\cite{howard2021cognitiveRadar}. These settings motivate cooperative multi-player bandits in which agents coordinate implicitly under information constraints, and the relevant information structure is dictated by what each agent can observe: actions, rewards, or both.

To model continuously many actions we work with Lipschitz bandits, where the mean reward varies smoothly over a metric space so that nearby actions have similar rewards. Discretization-based methods exploit this structure, but the resolution they should use depends on the Lipschitz constant~$L$. We therefore study the setting in which $L$ is \emph{not} given in advance, extending the multiplayer Lipschitz bandits of~\cite{chang2025metric}. Three forms of information asymmetry, described in Section~\ref{subsec:extension} and referred to as Problems~A, B, and~C, are treated in turn: unobserved actions with common rewards, observed actions with independent rewards, and unobserved actions with independent rewards. Because players cannot communicate once learning begins, an unknown $L$ creates a difficulty absent in the single-agent problem: each player estimates $L$ from its own observations, and if these estimates disagree the players build \emph{different} grids and can no longer be regarded as playing a common discrete bandit. Coordination, rather than estimation accuracy, is the binding constraint.

\textbf{Contributions.} We give a meta-algorithm, \texttt{mECAB}, that estimates an upper confidence bound on $L$ by uniform exploration of a coarse grid, fixes a discretization, and runs an existing cooperative multiplayer MAB subroutine; we prove regret bounds for all three information structures, each resting on a different agreement mechanism, namely shared rewards in Problem~A, signalling through observed actions in Problem~B, and a dithered quantization of the estimate in Problem~C, the last of which makes the agreement probability independent of the instance in a way no deterministic rounding rule can; and we report simulations comparing Lipschitz-adaptive with non-adaptive discretization.

\subsection{Related Work}

The Lipschitz bandit literature originates with the continuum-armed bandit (CAB) model~\cite{agrawal1995continuum}, in which arms lie in a continuum and the mean reward obeys a continuity condition; sharper discretization strategies followed~\cite{kleinberg2004nearly}, and the zooming algorithm~\cite{kleinberg2008metric} treats general metric spaces through the covering dimension. Closest to our starting point is~\cite{bubeck2011lipschitzconstant}, which removes the need to know $L$ a priori by estimating it from a uniform grid; our estimator and notation follow that work.

Multi-player extensions introduce collisions and asymmetric feedback; see~\cite{boursier2024survey} for a survey. Recent work on cooperative learning under different information structures adapts UCB-type ideas to obtain near-optimal guarantees~\cite{chang2021multiplayer,chang2023optimal,chang2025metric}, and we use those algorithms as subroutines. To the best of our knowledge no existing multi-player algorithm addresses Lipschitz bandits with an unknown Lipschitz constant; the closest work is~\cite{bistritz2020cooperative}, which assumes a maximal Lipschitz constant but does not learn~$L$.

\section{Preliminaries}\label{sec:preliminary}

Consider a $d$-dimensional compact set of arms $X = [0,1]^{d}$, $d \geq 1$. Each arm $\bm{x}$ carries a reward distribution $\nu_{\bm{x}}$ supported on $[0,1]$, with mean-payoff function $f : [0,1]^{d} \to [0,1]$. At each round $t \geq 1$ the player selects $\bm{x}_{t}$ and receives $Y_{t} \sim \nu_{\bm{x}_{t}}$, drawn independently across rounds.

\textbf{Assumption.} $f$ is twice differentiable with Hessian uniformly bounded by $N$: for all $\bm{x},\bm{y}$, $|\bm{y}^{\top} H_{f}(\bm{x})\,\bm{y}| \leq N\|\bm{y}\|_{\infty}^{2}$. The map $\|\grad f\|_1$ is continuous and attains its maximum $L$ on $[0,1]^{d}$, so $f$ is $L$-Lipschitz with respect to $\|\cdot\|_{\infty}$. Write $\mathcal{F}_{L,N}$ for the class of mean-payoff functions satisfying both conditions with parameters $L$ and $N$; all suprema below are over this class, and neither parameter is known to the players.

With $f^{\star} = \max_{\bm{x}} f(\bm{x})$, the expected regret at horizon $T$ is $R_{T} = \mathbb{E}[Tf^{\star} - \sum_{t=1}^{T} f(\bm{x}_t)]$, the expectation being over the draws of $Y_t$ and any internal randomization. The goal is to minimize $R_T$ without knowing $L$.

\section{Main Results}\label{sec:main}
\subsection{Extension to the Multi-Agent Setting}
\label{subsec:extension}

Let $P_{1},\dots,P_{M}$ be players, each holding a $d$-dimensional set of arms $A_{i} = [0,1]^{d}$, so the joint action space is $\mathcal{A} = A_{1}\times\cdots\times A_{M} = [0,1]^{Md}$, with joint arms $\bm{a} = (a_1,\dots,a_{Md})$. Players may agree on a strategy, and on any shared randomness it uses, before learning begins, but cannot communicate afterwards. At each round every player chooses $\bm{a}^{i}_{t}\in[0,1]^{d}$ simultaneously, forming $\bm{a}_{t}$; with $f : \mathcal{A}\to[0,1]$ in $\mathcal{F}_{L,N}$, the objective is again to minimize $R_T$ when $L$ is unknown.

\textbf{Problem A: unobserved actions, common rewards.} Every player receives the same reward $Y_{t}$ but does not observe the actions of the others.

\textbf{Problem B: observed actions, independent rewards.} Every player observes the actions of all others, but rewards are i.i.d.\ across players and player~$i$ sees only its own reward $Y_t^i$. The regret $R_T^i = \mathbb{E}[Tf^{\star} - \sum_{t} f(\bm{a}_t)]$ does not depend on $i$, since the rewards are identically distributed.

\textbf{Problem C: unobserved actions, independent rewards.} The two difficulties combine: rewards are i.i.d.\ across players and actions are unobserved.

\subsection{The Main Algorithm}

The algorithm is motivated by~\cite{bubeck2011lipschitzconstant} and follows the classical CAB template of~\cite{kleinberg2004nearly}: use the Lipschitz constant to discretize the space, then run a standard MAB algorithm on the resulting finite set. When $L$ is known this yields sublinear regret; when it is unknown we must first estimate it.

During exploration each player splits its own action set into $m^d$ bins, inducing $m^{Md}$ joint bins, and estimates an upper bound on $L$ from the differences between neighboring bins. Write $\overline{L}_m$ for the expectation of the estimator of~\cite{bubeck2011lipschitzconstant}, which approaches $L$ as $m$ grows.

\begin{lemma}[Bubeck et al.~\cite{bubeck2011lipschitzconstant}]\label{lemma:bias}
For $m\geq 3$, $\;L-\frac{7N}{m} \leq \overline{L}_m \leq L$, where $N$ is the Hessian bound.
\end{lemma}

\begin{lemma}\label{lemma:concentration}
If each joint bin is explored with $E'$ samples, then with probability at least $1-\delta$,
$\bigl|\widehat{L}_m-\overline{L}_m\bigr| \leq m \sqrt{\frac{2}{E'} \ln \frac{2 m^{Md}}{\delta}}$.
\end{lemma}

Adding the deviation of Lemma~\ref{lemma:concentration} to $\widehat{L}_m$ produces an upper confidence bound $\widetilde{L}$ on $\overline L_m$, which sets the discretization $\widetilde{m}$; combining the two lemmas with $\delta = 1/T$ gives the two-sided control used in all of the proofs.

\begin{corollary}\label{cor:Lbounds}
Fix $\delta = 1/T$ and let $E'$ denote the number of samples per joint bin available to a player, so $E' = ME$ in Problem~B and $E' = E$ in Problems~A and~C. With probability at least $1-1/T$,
\[
\overline{L}_m \;\leq\; \widetilde{L}_m \;\leq\; L + 1 + 2m\sqrt{\tfrac{2}{E'}\ln\!\left(2m^{Md}T\right)},
\]
and if in addition $m \geq 8N/L$ then $\widetilde{L}_m \geq L/8 - 1$.
\end{corollary}

The lower bound prevents the algorithm from choosing too coarse a grid and is the only place where the Hessian bound $N$ enters: the condition $m \geq 8N/L$ asks that the coarse grid already resolve the curvature of $f$, and it holds for all large $T$ under the choice of $m$ made in the proofs. Any multiplayer MAB algorithm, for instance~\cite{chang2021multiplayer,chang2023optimal}, is then run on the discretized joint space; Algorithm~\ref{alg} collects the steps.

\begin{algorithm}[t]
\caption{\texttt{mECAB}}
\label{alg}
\KwIn{Horizon $T$, coarse bins $m$, exploration budget $E$, dimension $d$.}
\textbf{Initialize:} each player divides $A_i = [0,1]^{d}$ into $m^{d}$ bins, inducing joint bins $\underline{k}\in\{0,\dots,m^{Md}-1\}$.\\
\textbf{Pre-learning:} players agree on an ordering of the joint bins, and in Problem~C on a dither $U\sim\mathrm{Unif}[0,1)$.\\
\textbf{Exploration:}\\
\For{each joint bin $\underline{k} \in \{0,\dots,m^{Md}-1\}$}{
    Each player samples $E$ actions uniformly from its own bin and observes the resulting rewards.\label{explore}\\
    Compute the empirical bin mean $\widehat{\mu}_{\underline{k}}$ (resp.\ $\widehat{\mu}^{\,i}_{\underline{k}}$).
}
Form $\widehat{L}$ from \eqref{discretization_A}, \eqref{discretization_B}, or \eqref{discretization_C} according to the problem, and set
\begin{equation}\label{eq:mtilde}
\widetilde{L} = \widehat{L}+ m\sqrt{\tfrac{2}{E'}\ln (2m^{Md}T)}, \;\;
\widetilde{m} = \bigl\lceil \widetilde{L}^{\frac{2}{Md+2}} T^{\frac{1}{Md+2}}\bigr\rceil,
\end{equation}
with $E' = ME$ in Problem~B and $E' = E$ otherwise.\label{discretization}\\
\textbf{Exploitation:}\\
\For{$t = E\,m^{Md} + 1$ \KwTo $T$}{
    Play the multiplayer MAB subroutine on the $\widetilde{m}^{Md}$ joint actions.
}
\end{algorithm}

\section{Problem A: Action Information Asymmetry}

Here the environment produces a single common reward observed by all players, while each player does not observe the others' actions. Unobserved actions prevent coordination during play, and the unknown smoothness must be estimated without access to the exploration of the others. What rescues the situation is that the reward is shared, which makes the exploration statistics shared as well once the schedule is fixed in advance.

Concretely, the players agree on an ordering of the $m^{Md}$ joint bins and each samples uniformly inside the scheduled bin. Only the bin index affects the statistic collected, not the particular arm chosen inside it, so all players obtain the same empirical mean
$\widehat{\mu}_{\underline{k}} = \frac{1}{E}\sum_{j=1}^{E} Z_{\underline{k},j}$
for every joint bin $\underline{k}$, and each is free to choose any arm within its own bin. Consequently every player forms the same estimate
\begin{equation}\label{discretization_A}
 \widehat{L} = m \max_{k \in[m-2]^{Md},\; s \in \{-1,1\}^{Md}} \bigl|\widehat{\mu}_k - \widehat{\mu}_{k+s}\bigr|,
\end{equation}
hence the same $\widetilde{L}$ and $\widetilde{m}$, and the ordering of the coarse bins induces a consistent ordering of the finer $\widetilde{m}$-level grid that everyone can follow.

\begin{theorem}\label{regret_A}
Let $m \geq 8N/L$ and use \texttt{m-UCB} of~\cite{chang2021multiplayer} with $\widehat{L}$ from~\eqref{discretization_A}. Then \texttt{mECAB} satisfies
\begin{multline*}
    \sup_{\mathcal{F}_{L,N}}R_{T}
    \leq T^{(Md+1)/(Md+2)} \\
    \cdot \left(9L^{Md/(Md+2)} + 5\left(2m\sqrt{\tfrac{2}{E}\ln{\left(2T^{Md+1}\right)}}\right)^{Md/(Md+2)}\right)\\
    + Em^{Md}+32\sqrt{T\widetilde{m}^{Md}\log T} + 1.
\end{multline*}
\end{theorem}

On the event of Corollary~\ref{cor:Lbounds} the subroutine term is $O(T^{(Md+1)/(Md+2)}(L{+}1)^{Md/(Md+2)}\sqrt{\log T})$, matching the leading term up to $\sqrt{\log T}$; the same holds in Theorems~\ref{regret_B} and~\ref{regret_C}.

In the single-agent problem with unknown $L$, discretization yields the familiar scaling $T^{(d+1)/(d+2)}$, and Problem~A behaves the same way on the joint space with $Md$ in place of $d$. The action asymmetry therefore costs nothing beyond this dimensional effect, precisely because the common reward and the pre-agreed schedule force identical bin means, an identical $\widehat{L}$, and an identical grid. The dependence on the size of the discretized joint set is also unavoidable, since with $K$ actions each the induced finite problem has $K^M$ joint arms, to which the standard finite-armed lower bound applies.

\section{Problem B: Reward Information Asymmetry}

Problem~B reverses the structure: actions are observable but reward observations are not shared. The free synchronization of Problem~A breaks, since the bin means $\widehat{\mu}^{\,i}_k$ may differ across players and would lead to different grids if nothing further were done.

The remedy is to exploit action observability to share reward information implicitly. Since every action is observed by everyone, an action can carry a signal encoding the sender's statistics, and in a continuum this needs no departure from the action space: after collecting $E-1$ samples from a bin, a player devotes its final action in that bin to encoding its empirical mean, which the others decode and fold into their own estimate. Nothing analogous exists in Problem~A, where actions are hidden, nor in finite-action models, where there is no room to encode a real number without distorting the learning problem. At a cost of one sample, each player thus gains $(M-1)(E-1)$ further samples for every bin, and forms
\begin{equation}\label{discretization_B}
 \widehat{L} = m \max_{k \in[m-2]^{Md},\; s \in \{-1,1\}^{Md}} \left|\frac{1}{M}\sum_{i=1}^M \bigl(\widehat{\mu}_k^i - \widehat{\mu}_{k+s}^i\bigr)\right|.
\end{equation}
The effective sample size entering the concentration of $\widehat{L}$ is multiplied by $M$, which sharpens the estimate and improves the grid. Problem~A had perfect alignment of the bin means but no way to convey anything beyond the common scalar reward; Problem~B lacks common rewards but recovers most of the benefit of centralized averaging by broadcasting estimates through actions.

\begin{theorem}\label{regret_B}
Let $m \geq 8N/L$ and use the multiplayer subroutine of~\cite{chang2023optimal} for Problem~B. Then \texttt{mECAB} satisfies
\begin{multline*}
    \sup_{\mathcal{F}_{L,N}} R_{T}
    \leq T^{(Md+1)/(Md+2)}\\
    \times \Bigl(9L^{Md/(Md+2)} \\
    \qquad + 5\Bigl(2m\sqrt{\tfrac{2}{ME}\ln{(2T^{Md+1})}}\Bigr)^{Md/(Md+2)}\Bigr)\\
    + Em^{Md}+ 32\sqrt{T\widetilde{m}^{Md}\log T} + 1.
\end{multline*}
\end{theorem}

Compared with Theorem~\ref{regret_A} the sample size inside the square root improves from $E$ to $ME$, which is exactly the pooling gain from the other players' samples.

\section{Problem C: Reward and Action Information Asymmetry}

Problem~C is the hardest of the three: rewards are not shared, so the mechanism of Problem~A is unavailable, and actions are not observed, so the signalling of Problem~B is unavailable too. If each player simply used~\eqref{discretization_A} on its own data, the estimates would differ and the induced grids $\widetilde{m}$ would differ with them, destroying the common discrete problem the subroutine needs.

We restore agreement by quantizing the estimate, so that small discrepancies between players do not change the value they act on. Let $X^i = m \max_{k,s} |\widehat{\mu}_k^i - \widehat{\mu}_{k+s}^i|$ be the raw estimate of player $i$, the maximum running over $k \in[m-2]^{Md}$ and $s \in \{-1,1\}^{Md}$ as above. A deterministic rounding will not do: if $\overline{L}_m$ sits near a rounding boundary, two players whose estimates straddle it round differently however many samples they collect, and the failure probability approaches $1/2$ regardless of $E$. The boundaries are fixed while $\overline{L}_m$ is a property of the instance, so no deterministic rule avoids this. Instead the players agree in advance on a dither $U \sim \mathrm{Unif}[0,1)$, shared randomness requiring no communication, and set
\begin{equation}\label{discretization_C}
 \widehat{L}^{\,i} = \bigl\lfloor X^i + U \bigr\rfloor .
\end{equation}
Randomizing the offset makes the distance from $\overline{L}_m$ to the nearest boundary uniform rather than instance-dependent, so the probability of disagreement can be bounded with no reference to where $\overline{L}_m$ lies.

\begin{lemma}\label{lemma:bad}
For any $\delta > 0$ and any player $i$,
\[
P\bigl(|X^i - \overline{L}_m| > \delta\bigr) \leq 4(2m)^{Md}\exp\!\left(-\frac{E\delta^2}{32m^2}\right).
\]
Consequently, with $\widehat{L}^{\,i}$ as in~\eqref{discretization_C} and $A := 4M(2m)^{Md}$, the probability that the $M$ players do not all obtain the same value of $\widehat{L}$ is at most $17\,m\sqrt{\ln(A)/E}$.
\end{lemma}

When the players do agree they share the same $\widetilde{m}$, and Problem~C reduces to running the same discretized strategy as before; when they do not, they may follow different grids and we pay for that event in the regret.

\begin{theorem}\label{regret_C}
Let $m \geq 8N/L$, let $A = 4M(2m)^{Md}$, and use the multiplayer subroutine of~\cite{chang2021multiplayer} with $\widehat L$ from~\eqref{discretization_C}. Then
\begin{multline*}
\sup_{\mathcal{F}_{L,N}} R_{T}
    \leq T^{(Md+1)/(Md+2)} \\
    \times \left(9L^{Md/(Md+2)} + 5\left(2m\sqrt{\tfrac{2}{E}\ln{\left(2T^{Md+1}\right)}}\right)^{Md/(Md+2)}\right) \\
    + Em^{Md}+ C\log T\sqrt{T\widetilde{m}^{Md}} + 17\,Tm\sqrt{\ln(A)/E}.
\end{multline*}
In particular, if $E \geq m^{2}T^{2/(Md+2)}\ln A$ the last term is at most $17\,T^{(Md+1)/(Md+2)}$.
\end{theorem}

Problem~C therefore isolates what each kind of observability buys. Common rewards make $\widehat{L}$ shared automatically and no agreement term is needed; observable actions permit pooling and sharpen the concentration, improving $E$ to $ME$; when neither is available, agreement must be built from concentration together with dithered quantization. Its price is the final term, which is instance-independent and, for the stated exploration budget, of the same order as the leading one, so Problem~C matches Problems~A and~B up to constants once $E$ is large enough.

\section{Experiments}

We simulate a cooperative bandit with $M=2$ players and action dimension $d=1$ each, so $Md = 2$, over $T = 10^5$ rounds and $10$ independent trials for each configuration. A maximizer $a^\star$ is drawn uniformly from $[0,1]^{Md}$ once in each trial and the mean reward is $f(a) = -L\|a-a^\star\|_\infty$, which is $L$-Lipschitz with respect to $\ell_\infty$ and satisfies $f(a^\star) = 0$; rewards are Gaussian with unit variance. We report cumulative pseudo-regret $\sum_{t}(f(a^\star) - f(a_t))$ averaged over trials, with $\pm 1$ standard deviation shading.

Both algorithms discretize the joint space and run UCB on the resulting grid, differing only in how the resolution is set. \emph{Est-$L$} explores a coarse grid of $m$ bins in each coordinate with $E$ uniform samples in each bin, estimates $\widehat{L}_b = m\max_{(b,b')\in\mathcal{N}}|\widehat{\mu}(b)-\widehat{\mu}(b')|$ over neighboring bin pairs $\mathcal{N}$, pads it as in~\eqref{eq:mtilde}, and sets $\widetilde{m}$ from the result; this phase is not aimed at collecting reward, so regret grows roughly linearly while it runs. \emph{No-$L$} skips exploration and takes $\widetilde{m} = \lceil T^{1/(Md+2)}\rceil$, avoiding the up-front cost but risking a resolution mismatched to the smoothness of $f$.

The three information structures are modelled at the level of the feedback reaching each player rather than through the signalling and quantization mechanisms themselves: Problem~A supplies a common reward, Problem~B pools the $M$ independent rewards of a round, the idealized effect of encoding empirical means in actions, and Problem~C uses a single reward stream without pooling. This isolates the effect of feedback quality on the accuracy of $\widetilde{L}$; simulating the signalling and dithering steps directly is left to an extended version.

Fig.~\ref{fig:cum-regret} compares the two rules for a small ($L=1$) and a large ($L=1000$) Lipschitz constant, everything else held fixed. In all three cases the Est-$L$ curves grow linearly during coarse exploration and then bend into a visibly sublinear phase once UCB begins on the refined grid, which is the tradeoff the method makes. When $L$ is small the two rules produce comparable resolutions and finish at similar levels; when $L$ is large, fixing the resolution without reference to $L$ gives a grid too coarse for the variation of $f$, and Est-$L$ overtakes it despite the initial linear segment. The information structure modulates the gain: pooling in Problem~B makes $\widetilde{L}$ more accurate and flattens the later slope relative to Problems~A and~C, while Problem~C, with the weakest feedback, is the most variable.

\begin{figure}[t]
  \centering
  \includegraphics[width=\columnwidth]{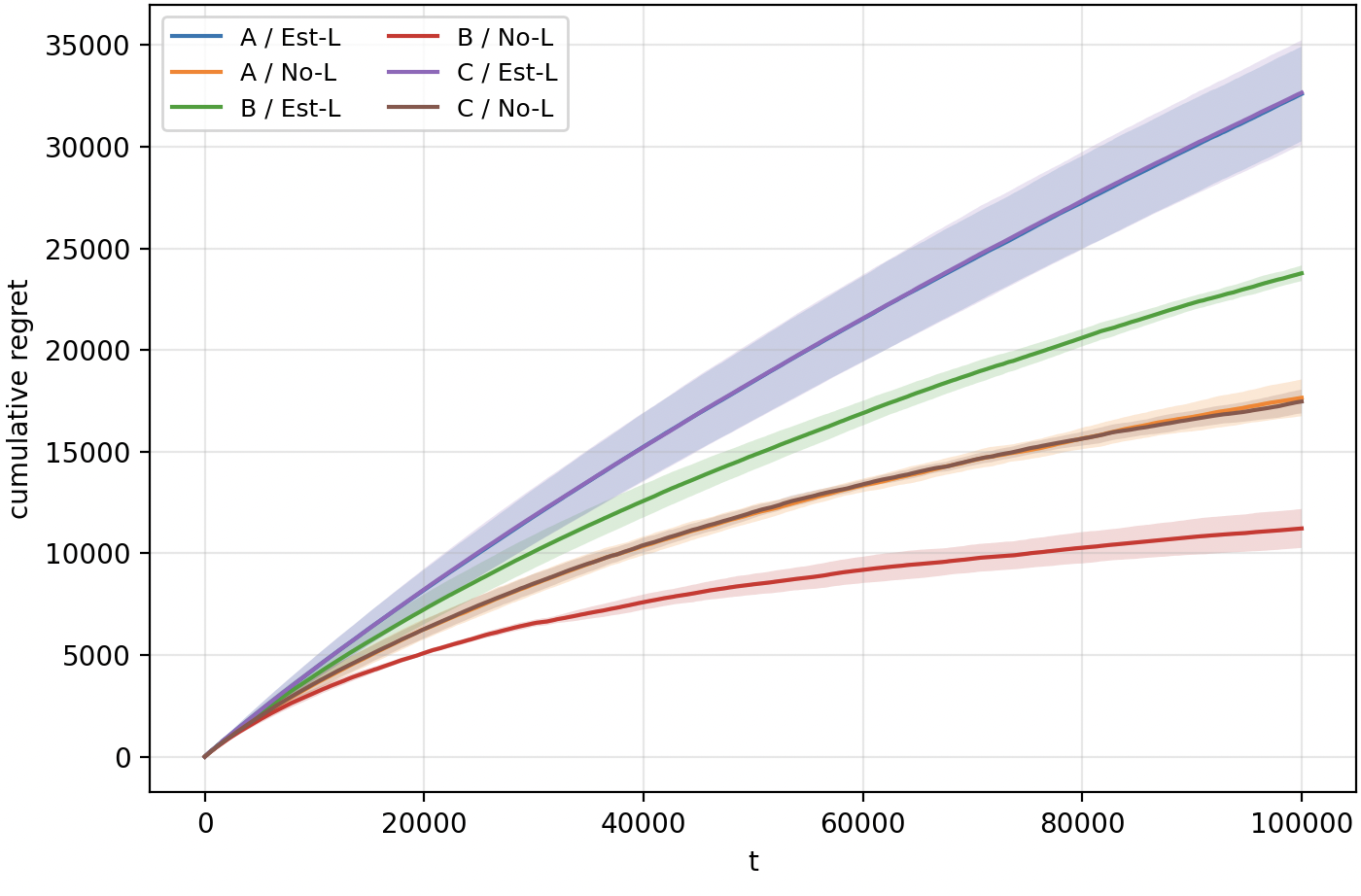}\\[2pt]
  \includegraphics[width=\columnwidth]{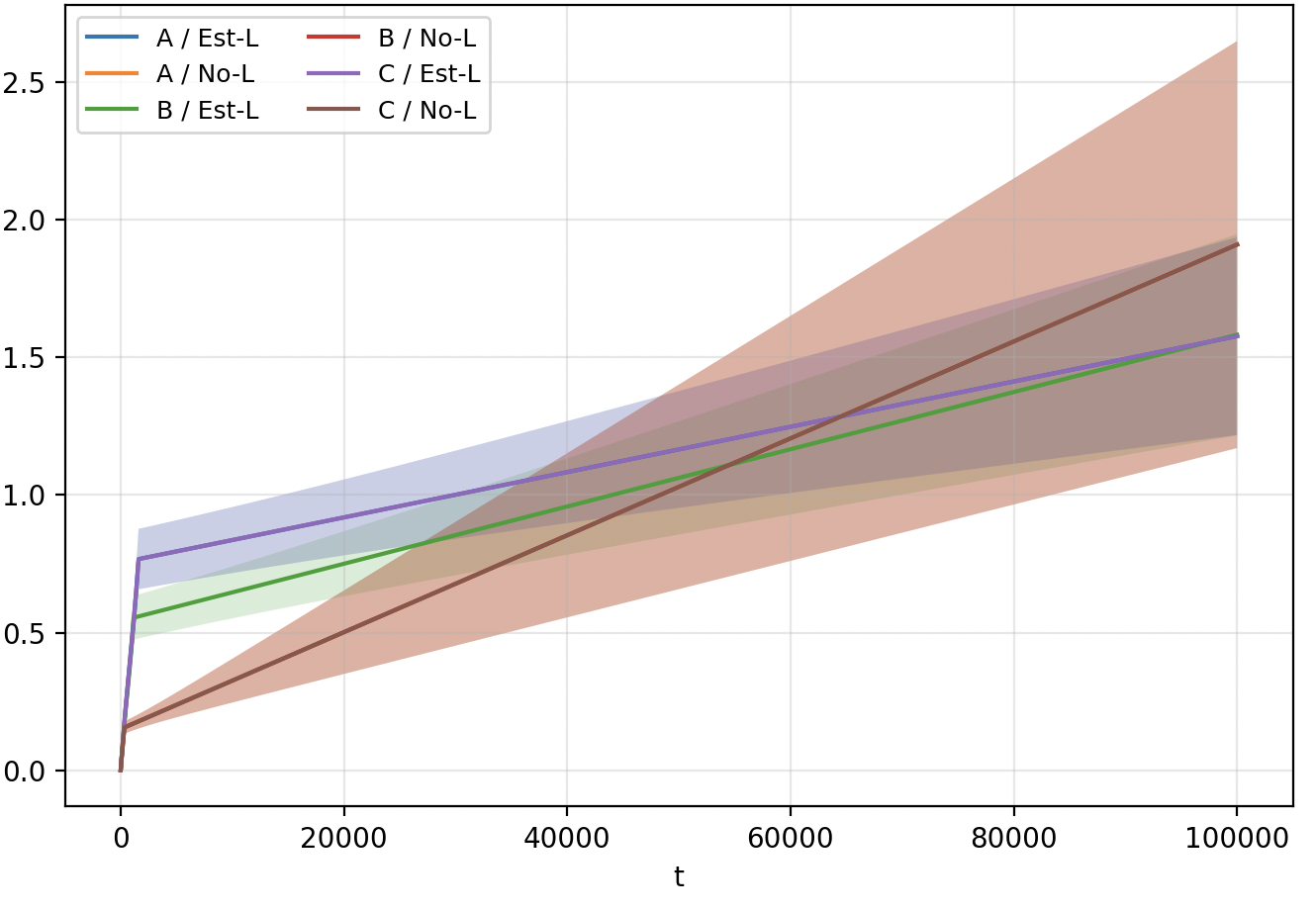}
  \caption{Cumulative regret averaged over trials with $\pm1$ standard deviation shading, for a small Lipschitz constant ($L{=}1$, top) and a large one ($L{=}1000$, bottom). Each panel overlays Problems~A, B, and~C under Est-$L$ and No-$L$ discretization.}
  \label{fig:cum-regret}
\end{figure}

\section{Conclusion}

We extended cooperative multiplayer bandits to Lipschitz action spaces with an unknown Lipschitz constant, where the players' estimates of $L$ must agree for a common discretization to exist. Common rewards and observable actions each deliver that agreement for free, and when neither is present a dithered quantization delivers it at no cost in the leading order of the regret. Natural next steps are adversarial rewards and structural assumptions beyond Lipschitz continuity.

\appendix

\subsection{Proof of Lemma~\ref{lemma:bad}}
\begin{IEEEproof}
Fix $\delta>0$ and let $\overline{f}_m(k)$ denote the mean of $f$ over bin $k$. By Hoeffding's inequality, for any bin $k$,
\begin{equation}\label{eq:hoef}
    P\!\left(|\widehat{\mu}_k - \overline{f}_m(k)| > \tfrac{\delta}{2m}\right) \leq 2\exp\!\left(-\frac{E\delta^2}{32m^2}\right).
\end{equation}
For a neighboring pair $(k,k')$, the triangle inequality and~\eqref{eq:hoef} give
\begin{align*}
    &P\!\left(\left||\widehat{\mu}_k - \widehat{\mu}_{k'}| - |\overline{f}_m(k) - \overline{f}_m(k')|\right| > \tfrac{\delta}{m}\right) \\
    &\quad \leq P\!\left(|\widehat{\mu}_k - \overline{f}_m(k)| > \tfrac{\delta}{2m}\right) + P\!\left(|\widehat{\mu}_{k'} - \overline{f}_m(k')| > \tfrac{\delta}{2m}\right) \\
    &\quad \leq 4\exp\!\left(-\frac{E\delta^2}{32m^2}\right).
\end{align*}
There are at most $m^{Md}2^{Md} = (2m)^{Md}$ pairs $(k,s)$ in the maximum defining $X^i$, so a union bound over them and multiplication by $m$ yield the first claim.

For the second, let $\Delta_U$ be the distance from $\overline{L}_m + U$ to the nearest integer; since $U$ is uniform on $[0,1)$, $\Delta_U$ is uniform on $[0,\tfrac12]$. If $|X^i - \overline{L}_m| < \Delta_U$ for every $i$, then all the $X^i + U$ lie in the same unit interval and every player obtains the same $\widehat{L}$. Writing $a = E/(32m^2)$ and $A = 4M(2m)^{Md}$, a union bound over the $M$ players and the first claim give, conditionally on $U$, a disagreement probability of at most $\min\{1, A e^{-a\Delta_U^2}\}$. Let $\delta_0 = m\sqrt{32\ln(A)/E}$, so that $Ae^{-a\delta_0^2}=1$. Averaging over $U$, whose density is $2$ on $[0,\tfrac12]$,
\begin{align*}
    P(\text{disagreement}) &\leq 2\delta_0 + 2A\!\int_{\delta_0}^{\infty}\! e^{-a\delta^2}\,d\delta \\
    &\leq 2\delta_0 + \frac{A e^{-a\delta_0^2}}{a\delta_0} = 2\delta_0 + \frac{32m^2}{E\delta_0} \\
    &\leq 3\delta_0 \leq 17\,m\sqrt{\ln(A)/E},
\end{align*}
where the second inequality uses $\int_{\delta_0}^{\infty} e^{-a\delta^2}d\delta \leq e^{-a\delta_0^2}/(2a\delta_0)$ and the third uses $\ln A \geq 1$.
\end{IEEEproof}

\subsection{Proof of Theorems~\ref{regret_A}, \ref{regret_B} and~\ref{regret_C}}
\begin{IEEEproof}
We give the argument once, writing $E'$ for the samples per joint bin available to a player, so $E' = ME$ in Problem~B and $E' = E$ otherwise, and $\mathcal{R}(K,T)$ for the regret of the multiplayer subroutine on $K$ joint arms. Exploration costs at most $Em^{Md}$, the discretization bias costs $LT/\widetilde{m}$, and the subroutine contributes $\mathcal{R}(\widetilde{m}^{Md},T)$, so
\begin{equation}\label{eq:skeleton}
\sup_{\mathcal{F}_{L,N}} R_{T} \leq Em^{Md} + \mathbb{E}\!\left[\frac{LT}{\widetilde{m}} + \mathcal{R}(\widetilde{m}^{Md},T)\right] + \Xi,
\end{equation}
where $\Xi = 0$ in Problems~A and~B, since all players hold the same $\widetilde{m}$ by construction, and $\Xi = T\cdot 17m\sqrt{\ln(A)/E}$ in Problem~C by Lemma~\ref{lemma:bad}, bounding the regret on the disagreement event by $T$.

Writing $x = \widetilde{L}_m^{2/(Md+2)}T^{1/(Md+2)}$, \eqref{eq:mtilde} gives $\widetilde{m} \leq x(1 + 1/x)$, and when $x \geq Md$ we may use $(1+1/x)^{Md}\leq e$. Substituting into~\eqref{eq:skeleton},
\begin{align*}
    \sup_{\mathcal{F}_{L,N}} R_{T} \leq\; & Em^{Md} + \Xi \\
    & + \mathbb{E}\bigg[T^{(Md+1)/(Md+2)}\frac{L+1}{\widetilde{L}_m^{2/(Md+2)}} \\
    & \quad + C\log T\sqrt{Te\bigl(T^{1/(Md+2)}\widetilde{L}_m^{2/(Md+2)}\bigr)^{Md}}\bigg].
\end{align*}
By Corollary~\ref{cor:Lbounds}, which applies since $m \geq 8N/L$, with probability at least $1-1/T$,
\begin{align*}
    &C\log T\sqrt{Te\bigl(T^{1/(Md+2)}\widetilde{L}_m^{2/(Md+2)}\bigr)^{Md}} \\
    &\quad \leq C\log T\sqrt{e}\,T^{(Md+1)/(Md+2)} \\
    &\qquad \cdot\!\left((L{+}1)^{\frac{Md}{Md+2}} + \!\left(2m\sqrt{\tfrac{2}{E'}\ln(2m^{Md}T)}\right)^{\!\frac{Md}{Md+2}}\right)\!,
\end{align*}
while $\widetilde{L}_m \geq L/8 - 1$ controls the first term; on the complementary event, of probability below $1/T$, the regret is at most $T$ and contributes at most $1$. Bounding $m$ by $T$ inside the logarithm and choosing $m = \lfloor T^{\alpha}\rfloor$ and $E = m^{2M}\lceil T^{2\gamma(Md+2)/(Md)}\rceil$ for suitable $\alpha,\gamma>0$ gives the stated bounds, with $\mathcal{R}(\widetilde{m}^{Md},T) = 32\sqrt{T\widetilde{m}^{Md}\log T}$ for Problem~A by~\cite{chang2021multiplayer} and for Problem~B by~\cite{chang2023optimal}, and $\mathcal{R}(\widetilde{m}^{Md},T) = C\log T\sqrt{T\widetilde{m}^{Md}}$ for Problem~C by~\cite{chang2021multiplayer}. For the final claim of Theorem~\ref{regret_C}, $E \geq m^2T^{2/(Md+2)}\ln A$ gives $17Tm\sqrt{\ln(A)/E} \leq 17T^{(Md+1)/(Md+2)}$.
\end{IEEEproof}

\bibliographystyle{IEEEtran}
\bibliography{references}

\end{document}